\documentclass[11pt]{article}

\usepackage{fullpage,times}

\usepackage{amsthm,amsfonts,amsmath,amssymb,epsfig,color,float,graphicx,verbatim}
\usepackage{algorithm,algorithmic}

\usepackage{hyperref}
\hypersetup{
	colorlinks   = true, 
	urlcolor     = blue, 
	linkcolor    = blue, 
	citecolor   = black 
}

\newtheorem{theorem}{Theorem}

\newtheorem{lemma}{Lemma}
\newtheorem{corollary}{Corollary}

\newtheorem{remark}{Remark}

\renewcommand{\epsilon}{\varepsilon}
\renewcommand{\tilde}{\widetilde}
\newcommand{\reals}{\mathbb{R}}
\newcommand{\E}{\mathbb{E}}
\renewcommand{\Pr}{\mathbb{P}}

\newcommand{\bw}{\mathbf{w}}

\newcommand{\bu}{\mathbf{u}}
\newcommand{\bv}{\mathbf{v}}

\newcommand{\bc}{\mathbf{c}}

\newcommand{\Ocal}{\mathcal{O}}
\newcommand{\Acal}{\mathcal{A}}

\newcommand{\Ccal}{\mathcal{C}}

\newcommand{\Wcal}{\mathcal{W}}

\newcommand{\loss}{\mathcal{L}}

\newcommand{\norm}[1]{\|#1\|}
\newcommand{\inner}[1]{\langle#1\rangle}

\newcommand{\secref}[1]{Sec.~\ref{#1}}

\renewcommand{\eqref}[1]{Eq.~(\ref{#1})}
\newcommand{\lemref}[1]{Lemma~\ref{#1}}

\newcommand{\thmref}[1]{Thm.~\ref{#1}}

\newcommand{\bone}{\boldsymbol{1}}
\newcommand{\bzero}{\boldsymbol{0}}
\renewcommand{\loss}{\ell}
\newcommand{\wh}{\widehat}
\newcommand{\hloss}{\wh{\loss}}
\newcommand{\tloss}{\widetilde{\loss}}
\newcommand{\bloss}{\boldsymbol{\ell}}
\newcommand{\tbloss}{\widetilde{\boldsymbol{\ell}}}

\title{\bf Bandit Regret Scaling with the Effective Loss Range}

\author{
  Nicol\`{o} Cesa-Bianchi \\
  Dipartimento di Informatica\\
  Università degli Studi di Milano\\
  Milano 20133, Italy \\
  \texttt{nicolo.cesa-bianchi@unimi.it} \\
  \and
  Ohad Shamir \\
  Department of Computer Science \\
  Weizmann Institute of Science \\
  Rehovot 7610001, Israel\\
  \texttt{ohad.shamir@weizmann.ac.il} \\
}

\date{}

\begin{document}

\maketitle

\bigskip
\begin{center}
	\relax
		\begin{tabular}{|c|}
			\hline
			~\parbox{0.9\linewidth}{\hfil\large\bf Erratum
			}~\\[1.5ex]
			~\parbox{0.9\linewidth}{\relax
				The results of \secref{sec:laplace}, except \thmref{thm:laplow}, are incorrect as stated, due to a crucial bug in the proof of \thmref{th:upper} (in particular, the claim in \lemref{l:minor} is not correct). However, the following modification of \thmref{th:upper} is still correct:
				\begin{theorem}\label{th:upper_fixed}
					Assume that in each round $t$, after choosing $I_t$ the learner is told a number $a_t\geq 0$ such that $\min_i \loss_t(i)\geq a_t$. Then Exp3 performing updates based on loss vectors 
					$\tbloss_t = \bloss_t -a_t \bone$ achieves
					\begin{align*}
					\E\left[\sum_{t=1}^{T} \loss_t(I_t)\right] - \min_{i=1,\dots,K}\sum_{t=1}^{T} \loss_t(i)  
					\le
					\frac{\log K}{\eta} + \frac{\eta}{2} \sum_{t=1}^{T}\sum_{i=1}^{K}(\ell_t(i)-a_t)^2~.
					\end{align*}
				\end{theorem}
			    The theorem is based on a simple shifting argument, and is an immediate corollary of \lemref{lem:exp3}
			    and the argument leading to \eqref{eq:convprog} in the proof of \thmref{th:upper}. Compared to \thmref{th:upper}, we: (1) Need to make the different assumption that we are told a lower bound on the losses (e.g., the smallest loss), as opposed to any loss; (2) Define $\tbloss_t$ a bit differently; and (3) The regret bound now depends on the variation of the losses through $\sum_i (\ell_t(i)-a_t)^2$ rather than through a Laplacian-based bound. Thus, we still get a regret bound which is always better than the standard Exp3 regret bounds, improves as the per-round losses become more similar, and becomes zero if all the losses are exactly the same (as it should be expected). Finally, \thmref{thm:laplow}, which provides a Laplacian-based lower bound, is still correct but is no longer matched by a similar upper bound. It remains open whether such an upper bound exists for a bandit setting similar to the one considered here.

\smallskip\noindent
Thanks to \'{E}tienne de Montbrun for finding the problem in \lemref{l:minor}.
			    \\[0.3ex] }~\\ \hline
	\end{tabular}
\end{center}

\newpage

\begin{abstract}
We study how the regret guarantees of nonstochastic multi-armed bandits can be 
improved, if the effective range of the losses in each round is small (e.g. the 
maximal difference between two losses in a given round). Despite a recent 
impossibility result, we show how this can be made possible under certain mild 
additional assumptions, such as availability of rough estimates of the losses,
or advance knowledge of the loss of a single, possibly unspecified arm. Along 
the way, we develop a novel technique which might be of independent interest, 
to convert any multi-armed bandit algorithm with regret depending on the loss 
range, to an algorithm with regret depending only on the effective range, while 
avoiding predictably bad arms altogether.
\end{abstract}

\section{Introduction}

In the online learning and bandit literature, a recent and 
important trend has been the development of algorithms which are capable of 
exploiting ``easy'' data, in the sense of improved regret guarantees if the 
losses presented to the learner have certain favorable patterns. For example, a 
series of works have studied how the regret can be improved if the losses do 
not change much across rounds (e.g., \cite{chiang2012online,hazan2010extracting,hazan2011better,NIPS2016_6341,steinhardt2014adaptivity});
 being simultaneously competitive w.r.t.\ both ``hard'' and ``easy'' data (e.g.,
\cite{seldin2014one,sani2014exploiting,auer2016algorithm,bubeck2012best}); 
attain richer feedback on the losses (e.g., \cite{alon2014nonstochastic}),  
have some predictable structure \cite{rakhlin2013online}, and so on.
In this paper, we continue this research agenda in a different direction, focusing on improved regret performance in nonstochastic settings with partial feedback where the learner has some knowledge about the variability of the losses \emph{within} each round.

In the full information setting, where the learner observes the entire set of 
losses $\loss_t(1),\dots,\loss_t(K)$ after each round $t$, it is possible to 
obtain regret bounds of order $\epsilon\sqrt{T\log K}$ scaling with the unknown 
\emph{effective range} $\epsilon = \max_{t,i,j} |\loss_t(i)-\loss_t(j)|$ of the 
losses~\cite[Corollary~1]{cesa2007improved}. Unfortunately, the situation in 
the bandit setting, where the learner only observes the loss of the chosen 
action, is quite different. A recent surprising 
result~\cite[Corollary~4]{gertor16} implies that in the bandit setting, the 
standard $\Omega(\sqrt{KT})$ regret lower bound holds, even when $\epsilon = 
\Ocal(\sqrt{K/T})$. The proof defines a process where losses are kept 
$\epsilon$-close to each other, but where the values oscillate unpredictably 
between rounds. Based on this, one may think that it is impossible to attain 
improved bounds in the bandit setting which depend on $\epsilon$, or some other 
measure of variability of the losses across arms. In this paper, we show the 
extent to which partial information about the losses allows one to circumvent 
the impossibility result in some interesting ways. We analyze two specific 
settings: one in which the learner can roughly estimate in advance the actual 
loss value of each arm, and one where she knows the exact loss of some 
arbitrary and unknown arm.  

In order to motivate the first setting, consider a scenario where the learner knows each arm's loss up to a certain precision (which may be different for each arm). For example, in the context of stock prices~\cite{hazan2009stochastic,abernethy2013hedge} the learner may have a stochastic model providing some estimates of the loss means for the next round. 
In other cases, the learner may be able to predict that certain arms are going 
to perform poorly in some rounds. For example, in routing the learner may know 
in advance that some route is down, and a large loss is incurred if that route 
is picked. Note that in this scenario, a reasonable algorithm should be able to 
avoid picking that route. However, that breaks the regret guarantees of 
standard expert/bandit 
algorithms, which typically require each arm to be chosen with some positive 
probability. In the resulting regret bounds, it is difficult to avoid at 
least some dependence on the highest loss values. 

To formalize these scenarios and considerations, we study a setting where 
for each arm $i$ at 
round $t$, the learner is told that the loss will be in 
$[m_t(i)-\epsilon_t(i),m_t(i)+\epsilon_t(i)]$ for some $m_t(i),\epsilon_t(i)$. 
In 
this setting, we show a generic reduction, which allows one to convert any 
algorithm for bounded losses, under a generic feedback model (not necessarily 
a bandit one) to an algorithm with regret depending only on the effective range 
of 
the losses (that is, only on $\epsilon_t(i)$, independent of $m_t(i)$).
Concretely, taking the simple case where the loss of each arm $i$ at each round $t$ is in 
$[m_t(i)-\epsilon,m_t(i)+\epsilon]$ for some $m_t(i)$ and fixed $\epsilon$, and 
assuming the step 
size is properly chosen, we can get a regret bound of
$
\tilde{\Ocal}\big(\epsilon\sqrt{KT}\big)
$
for the bandit feedback, completely independent of $m_t(i)$ 
and the losses' actual range. Note that this has the desired behavior that as 
$\epsilon\rightarrow 0$, the regret also converges to zero (in the extreme case 
where $\epsilon=0$, the learner essentially knows the losses in advance, and 
hence can avoid any regret). With full information feedback (where the entire 
loss vector is revealed at the end of each round), we can use the same 
technique to recover the regret bound of 
$
\Ocal\big(\epsilon\sqrt{T\log K}\big)
$.
We note that this is a special case of the predictable sequences 
setting studied in~\cite{rakhlin2013online}, and their 
proposed algorithm and analysis 
is applicable here. However, comparing the results, our bandit regret bounds 
have a better dependence on the number of arms $K$, and our reduction can be 
applied to any algorithm, rather than the specific one proposed in  
\cite{rakhlin2013online}. On the flip 
side, the algorithm proposed in \cite{rakhlin2013online} is tailored to the 
more general setting of bandit linear optimization, and does not require the 
range parameter $\epsilon$ to be known in advance (see \secref{sec:meps} for a 
more detailed comparison). We also study the tightness of our regret guarantees 
by providing lower bounds.

A second scenario motivating partial knowledge about the loss vectors is the following. Consider a system for recommending products to visitors of some company's website. Say that two products are similar if the typical visitor tends to like them both or dislike them both. Hence, if we consider the similarity graph over the set of products, then it is plausible to assume that the likelihood of purchase (or any related index of the visitor's behavior) be a smooth function over this graph. Formally, the loss vectors $\boldsymbol{\ell}_t$ at each round $t$ satisfy $\boldsymbol{\ell}_t^\top L_t \boldsymbol{\ell}=\sum_{(i,j)\in E_t}(\ell_t(i)-\ell_t(j))^2\leq C_t^2$, where 
$L_t$ is the Laplacian matrix associated with a graph over the arms with edges 
$E_t$, and $C_t\geq 0$ is a smoothness parameter. In this setting, we provide 
improved bandit regret bounds depending on the spectral properties of the 
Laplacian. To circumvent the impossibility result of~\cite{gertor16} mentioned 
earlier, we make the reasonable assumption that at the end of each round 
round, the learner is given an ``anchor point'', corresponding to the loss of 
some unspecified arm. In our motivating example, the recommender system may 
assume, for instance, that each visitor has some product that she most likely 
won't buy. Using a simple modification of the Exp3 algorithm, we show that if 
the parameters are properly tuned, we attain a regret bound of order
$
\sqrt{\min\left\{KT,\sum_{t=1}^{T}\left(1+\frac{C_t^2}{\lambda_2(L_t)}
\right)\right\}}
$
(ignoring log factors),
where $\lambda_2(L_t)\in (0,K]$ is the second-smallest eigenvalue of $L_t$, also known as 
the algebraic connectivity number of the graph represented by $L_t$. If the 
learner is told the minimal loss at every round (rather than any loss), this 
bound can be improved to order of
$
\sqrt{\min\left\{KT,\sum_{t=1}^{T}\frac{C_t^2}{\lambda_2(L_t)}
\right\}}
$
(again, ignoring log factors) which vanishes, as it should, when $C_t = 0$ for all $t$; that is, when all arms share the same loss value.
We also provide a lower bound, showing that this upper bound is the best 
possible (up to log factors) in the worst case. Although our basic results pertain to 
connected graphs, using the range-dependent reductions discussed earlier we show it can be applied to graphs with multiple connected components and anchor points. 

The paper is structured as follows: In \secref{sec:setting}, we formally define 
the standard experts/bandit online learning setting, which is the focus of our 
paper, and devote a few words to the notation we use. In \secref{sec:meps}, we 
discuss the situation where each individual loss is known to lie in a certain 
range, and provide an algorithm as well as upper and lower bounds on the 
expected regret. In \secref{sec:laplace}, we consider the setting of smooth losses (as defined above).
All our formal proofs are presented in the appendices. 

\section{Setting and notation}\label{sec:setting}
The standard experts/bandit learning setting (with 
nonstochastic losses) is phrased as a 
repeated game between a learner and an adversary, defined over a fixed set of 
$K$ arms/actions. Before the game begins, the adversary assigns losses for 
each of $K$ arms and each of $T$ rounds (this is also known as an oblivious 
adversary, as opposed to a nonoblivious one which sets the losses during the 
game's progress). The loss of arm $i$ at round $t$ is 
defined as $\ell_t(i)$, and is assumed w.l.o.g.\ to lie in $[0,1]$. We let 
$\boldsymbol{\ell}_t$ denote the vector $(\ell_t(1),\ldots,\ell_t(K))$. At the 
beginning of each round, the learner chooses an arm $I_t\in \{1,\ldots,K\}$, 
and receives the associated loss $\ell_t(i)$. With bandit feedback, the 
learner then observes only her own loss $\ell_t(I_t)$, whereas with 
full information feedback, the learner gets to observe $\ell_t(i)$ for all $i$. 
The learner's goal is to minimize the expected regret (sometimes denoted as 
pseudo-regret), defined as
\[
\E\left[\sum_{t=1}^{T}\ell_t(I_t)\right]-\min_{i=1,\dots,K} 
\sum_{t=1}^{T}\ell_t(i)~,
\]
where the expectation is over the learner's possible randomness. We use 
$\mathbf{1}_A$ to denote the indicator of the event $A$, and let $\log$ 
denote the natural logarithm. Given an (undirected) graph over $K$ nodes, its 
Laplacian $L$ is defined as the $K\times K$ matrix where $L_{i,i}$ equals the 
degree of node $i$, $L_{i,j}$ for $i\neq j$ equals $-1$ if node $i$ is adjacent 
to node $j$, and $0$ otherwise. We let $\lambda_2(L)$ denote the 
second-smallest eigenvalue of $L$. This is also known as the algebraic connectivity 
number, and is larger the more well-connected is the graph. In particular, 
$\lambda_2(L)=0$ for disconnected graphs, and $\lambda_2(L)=K$ for the complete 
graph.

\section{Rough estimates of individual losses}\label{sec:meps}
We consider a variant of the online learning setting presented 
in \secref{sec:setting}, where at the beginning of every round $t$, the learner 
is provided with additional side information in the form of 
$\{m_{t}(i),\epsilon_{t}(i)\}_{i=1}^{K}$, with the guarantee that 
$|\ell_{t}(i)-m_{t}(i)|\leq \epsilon_t(i)$ for all $i=1,\ldots,K$. 
%
%
We then propose an algorithmic reduction, which allows to convert 
any regret-minimizing algorithm $\Acal$ (with some generic feedback), to an 
algorithm with regret depending on 
$\epsilon_t(i)$, independent of $m_t(i)$. We assume that given a loss vector 
$\bloss_t$ and chosen action $I_t$, the algorithm $\Acal$ receives as 
feedback some function $f_t\big(\bloss_t,I_t\big)$: For example, if $\Acal$ is 
an algorithm for the multi-armed bandits setting, then
$f_t\big(\bloss_t,I_t\big)=\bloss_t(I_t)$, whereas if $\Acal$ is an algorithm 
for the experts setting, $f_t\big(\bloss_t,I_t\big)=\bloss_t$. In our 
reduction, $\Acal$ is sequentially fed, at the end of each round $t$, with  
$f_t\big(\tilde{\bloss}_t,\tilde{I}_t\big)$ (where $\tilde{\bloss}_t$ and 
$\tilde{I}_t$ are not necessarily the same as the actual loss vector $\bloss_t$ 
and actual chosen arm $I_t$), and returns a recommended arm $\tilde{I}_{t+1}$ 
for the next round, which is used to choose the actual arm $I_{t+1}$.

To formally describe the reduction, we need a couple of definitions. For 
all $t$, let 
\[
j_t\in \arg\min_i \{m_t(i)-\epsilon_t(i)\}
\]
denote the arm with the lowest potential loss, based on the provided 
side-information (if there are ties, we
choose the one with smallest $\epsilon_t(i)$, and break any remaining ties 
arbitrarily). Define any arm $i$ as ``bad'' (at round $t$) if
	$
	m_{t}(i)-\epsilon_t(i) > m_{t}(j_t)+\epsilon_t(j_t)
	$
	and ``good'' if
	$
	m_{t}(i)-\epsilon_t(i) \leq m_{t}(j_t)+\epsilon_t(j_t)
	$.
	Intuitively, ``bad'' arms are those which cannot possibly have the smallest loss in round $t$. 
For loss vector $\bloss_t$, define the transformed loss vector $\tbloss_t$ as
	\[
	\tilde{\ell}_t(i)=\begin{cases}\ell_t(i)-m_t(j_t)+\epsilon_t(j_t)&\text{if $i$ is good}\\
	2\,\epsilon_t(j_t)&\text{if $i$ is bad.}\end{cases}
	\]
	It is easily verified that $\tilde{\ell}_t(i)\in 
	\big[0,2(\epsilon_t(i)+\epsilon_t(j_t))\big]$ always. Hence, the range of the 
	transformed losses does not depend on $m_t(i)$.
The meta-algorithm now does the following at every round:
\begin{enumerate}
	\item Get an arm recommendation $\tilde{I}_t$ from $\Acal$.
	\item Let $I_t=\tilde{I}_t$ if $\tilde{I}_t$ is a good arm, and $I_t=j_t$ otherwise.
	\item Choose arm $I_t$ and get feedback $f_t\big(\bloss_t,I_t\big)$
	\item Construct feedback $f_t\big(\tbloss_t,\tilde{I}_t\big)$ and feed to algorithm $\Acal$
\end{enumerate}
Crucially, note that we assume that 
$f_t\big(\tbloss_t,\tilde{I}_t\big)$ can be constructed based on 
$f_t\big(\bloss_t,I_t\big)$. For example, this is certainly true in the 
full information setting (as we are given $\bloss_t$, hence can 
explicitly compute $\tbloss_t$). This is also true in the 
bandit setting: If $\tilde{I}_t$ is a ``good'' arm, then $I_t=\tilde{I}_t$, 
hence we can construct 
$\tilde{\ell}_t(\tilde{I}_t)=\ell_t(I_t)-m_t(j_t)+\epsilon_t(j_t)$ 
based on the feedback $\ell_t(I_t)$ actually given to the meta-algorithm. If 
$\tilde{I}_t$ is a ``bad'' arm, then we can indeed construct 
$\tilde{\ell}_t(\tilde{I}_t) = 2\,\epsilon_t(j_t)$, since $\epsilon_t(j_t)$ is 
given to the meta-algorithm as side-information. This framework can potentially 
be used for other partial-feedback settings as well. 

The following key theorem implies that the expected regret of this 
meta-algorithm can be upper bounded by the expected regret of $\mathcal{A}$, 
with respect to the \emph{transformed} losses $\tbloss_t$ (whose range is 
independent of $m_t(i)$):
\begin{theorem}\label{thm:reduc}
	Suppose (without loss of generality) that $\tilde{I}_t$ given by $\Acal$ is 
	chosen at random by sampling from a probability distribution 
	$\tilde{p}_t(1),\ldots,\tilde{p}_t(K)$. Let $p_t(1),\ldots,p_t(K)$ be the 
	induced distribution\footnote{By definition of the meta-algorithm, we have 
	$p_t(i)=\tilde{p}_t(i)$ if $i\neq 
	j_t$ is good, $p_t(i)=0$ if $i$ is bad, and 
	$p_t(j_t)=\tilde{p}_t(j_t)+\sum_{\text{$i$ is bad}}\tilde{p}_{t}(i)$.} of $I_t$. 
	Then for any fixed arm $a\in \{1,\ldots,K\}$, it holds that
	\begin{equation}\label{eq:reduc}
	\sum_{t=1}^{T} \sum_{i=1}^{K} p_t(i)\ell_t(i)-\sum_{t=1}^{T} \ell_t(a)~\leq~\sum_{t=1}^{T}\sum_{i=1}^{K}\tilde{p}_t(i)\tilde{\ell}_t(i)-\sum_{t=1}^{T} \tilde{\ell}_t(a)~.
	\end{equation}
	This implies in particular that
	\[
	\E\left[\sum_{t=1}^{T}\ell_t(I_t)\right]-\sum_{t=1}^{T}\ell_t(a)~\leq~\E\left[\sum_{t=1}^{T}\Big(\tilde{\ell}_t(\tilde{I}_t)-\tilde{\ell}_t(a)\Big)\right]
	\]
	where the expectation is over the possible randomness of the algorithm $\Acal$.
	Moreover, $\tilde{\ell}_t(i)\in \big[0,2(\epsilon_t(j_t)+\epsilon_t(i))\big]$ for any good $i$, and $\tilde{\ell}_t(i)=2\,\epsilon_t(j_t)$ for any bad $i$.
\end{theorem}
The proof of the theorem (in the appendices) carefully relies on 
how the transformed losses and actions were defined. 
Since the range of $\tilde{\ell}_t$ is independent of $m_t$, we get a regret 
bound for our
meta-algorithm which depends only on $\epsilon_t$. This is exemplified in the 
following two corollaries:
\begin{corollary}\label{corr:reduc}
	With bandit feedback and using Exp3 as the algorithm $\Acal$ (with step 
	size $\eta$), the expected regret of the meta-algorithm is
	\[
	\Ocal\left(\frac{\log K}{\eta}+\eta\sum_{t=1}^{T}\left(K\epsilon_t(j_t)^2+\sum_{i\in
	 G_t}\epsilon_t(i)^2\right)\right)
	\]
	where $G_t\subseteq \{1,\ldots,K\}$ is the set of ``good'' arms at round 
	$t$. 
\end{corollary}
The optimal choice of $\eta$ leads to a regret of order 
$\sqrt{(\log K)\sum_{t=1}^{T}\left(K\epsilon_t(j_t)^2+\sum_{i\in 
G_t}\epsilon_t(i)^2\right)}$. This recovers the standard Exp3 bound in the 
case $m_t(i)=\epsilon_t(i)=\frac{1}{2}$ (i.e., the standard setting where the losses are 
only known to be bounded in $[0,1]$), but can be considerably better if the 
$\epsilon_t(i)$ terms are small, or the $m_t(i)$ terms are large. We also note 
that 
the 
$\log K$ factor can in principle be removed, i.e., by using the implicitly 
normalized forecaster of \cite{audibert2009minimax} with appropriate 
parameters. A similar corollary can be obtained in the full information 
setting, using a standard algorithm such as Hedge \cite{freund1995desicion}
\begin{corollary}\label{corr:reduc2}
	With full information feedback and using Hedge as
	the algorithm $\Acal$ (with step size $\eta$), the expected regret of the 
	meta-algorithm is
	\[
	\Ocal\left(\frac{\log K}{\eta}+\eta\sum_{t=1}^{T} \max_{i=1,\dots,K}\epsilon_t(i)^2\right).
	\]
\end{corollary}
The optimal choice of $\eta$ leads to regret of order 
$\sqrt{(\log K)\sum_{t=1}^{T}\max_i \epsilon_t(i)^2}$. As in the bandit setting, 
our reduction can be applied to other algorithms as well, including those with 
more refined loss-dependent guarantees (e.g.,
\cite{steinhardt2014adaptivity} and references therein).

Finally, we note that \thmref{thm:reduc} can easily be used to provide 
high-probability bounds on the actual regret 
$\sum_{t=1}^{T}\ell_t(I_t)-\sum_{t=1}^{T}\ell_t(a)$, rather than just bounds in 
expectation, as long as we have a high-probability regret bound for $\Acal$. 
This is due to \eqref{eq:reduc}, and can be easily shown using standard 
martingale arguments. 

\subsection{Related work}
As mentioned in the introduction, a question similar to those we are studying 
here was considered in \cite{rakhlin2013online}, under the name of learning 
with predictable sequences. Unlike our setting, however, 
\cite{rakhlin2013online} does not require knowledge of 
$\epsilon_t(i)$. Assuming the step size is chosen appropriately, 
they provide algorithms with expected regret bounds scaling as
\[
\begin{array}{lcl}
{\displaystyle \sqrt{(\log K)K^2\sum_{t=1}^{T}\sum_{i=1}^{K}\epsilon_t(i)^2} }
& \qquad\text{and}\qquad &
{\displaystyle \sqrt{(\log K)\sum_{t=1}^{T}\max_i \epsilon_t(i)^2} }
\\
\quad\text{(bandit feedback)} & & \quad\text{(full information feedback)}
\end{array}
\]
Comparing these bounds to Corollaries~\ref{corr:reduc} and~\ref{corr:reduc2}, we see that we obtain a similar regret 
bound in the full information setting, whereas in the bandit setting, our bound 
has a better 
dependence on the 
number of arms $K$, and better dependencies on 
$\epsilon_t(1),\ldots,\epsilon_t(K)$ if 
$\epsilon_t(j_t)$ or the number of ``good'' arms tends to be small. Also, our 
algorithmic approach is based on a reduction, which can be applied in principle 
to any algorithm and to general families of feedback settings, rather than a 
specific algorithm. On the flip side, our bound in the bandit setting can be 
worse than that of \cite{rakhlin2013online}, if $K\epsilon_t(j_t)^2 \gg 
\sum_{i}\epsilon_t(i)^2$. Also, their algorithm is tailored to the more general 
setting of linear bandits (where at each round the learner needs to pick a 
point $\bw_t$ in some convex set $\Wcal$, and receives a loss 
$\inner{\ell_t,\bc_t}$), and does not require knowing $\epsilon_t(i)$ in 
advance.

Another related line of work is path-based bounds, where it is assumed that the 
losses $\ell_t(i)$ tend to vary slowly with $t$, and $\ell_{t-1}(i)$ can 
provide a good estimate of $\ell_t(i)$. This can be linked to our setting by 
taking $m_t(i)=\ell_{t-1}(i)$, and $\epsilon_t(i)$ be some known upper bound 
on $|\ell_{t}(i)-\ell_{t-1}(i)|$. However, implementing this requires the 
assumption that $\ell_{t-1}$ is revealed at the next round $t$, which does not 
fit the bandit setting. Thus, it is difficult to directly compare these results 
to ours. Most work on this topic has focused on the full information feedback 
setting (see \cite{steinhardt2014adaptivity} and references therein), and the 
bandit setting was studied for instance in \cite{hazan2011better}. 

%

\subsection{Lower bound}
We now turn to consider the tightness of our results. Since the focus of this 
paper is to study the variability of the losses across arms, rather than across 
time, we will consider for simplicity the case where $\epsilon_t(j)$ are fixed 
for all $t=1,\ldots,T$ (hence the $t$ subscript can be dropped). 

In the theorem below, we show that the dependencies on $\sum_j \epsilon(j)^2$ 
and $\max_j \epsilon(j)$ (in the bandit and full information case, 
respectively) cannot be improved in general.
\begin{theorem}\label{thm:epslowbound}
Fix $T,K>1$ and nonnegative $\{\epsilon(i)\}_{i=1}^K$ such that $\min_{j \,:\,\epsilon(j)>0}\epsilon(j)^2 \geq \frac{2}{T}\sum_{j}\epsilon(j)^2$. 
 Then there exists fixed parameters $m(j)$ for $j=1,\ldots,K$ such that the 
 following holds: For any (possibly randomized) learner strategy $A$, there 
 exists a loss assignment satisfying $|\ell_t(j)-m(j)|\leq \epsilon(j)$ for all 
 $t,j$, such that
\[
	\E_A\left[\sum_{t=1}^{T}\ell_t(I_t)\right]-\min_{j=1,\dots,K}\sum_{t=1}^{T}\ell_t(j)
\geq \left\{ \begin{array}{cl}
	c\sqrt{T\sum_{j=1}^{K}\epsilon(j)^2} & \text{with bandit feedback}
\\[2mm]
	{\displaystyle c\sqrt{T\max_{j=1,\dots,K} \epsilon(j)^2} } & \text{with full information feedback}
\end{array} \right.
\]
where $c>0$ is a universal constant.
\end{theorem}
The proof is conceptually similar to the standard regret lower bound for 
nonstochastic multi-armed bandits (see \cite{bubeck2012regret}), where the 
losses are generated stochastically, with one randomly-chosen and hard-to-find 
arm having a slightly smaller loss in expectation. However, we utilize a more 
involved stochastic process to generate the losses as well as to choose the 
better arm, which takes the values of $\epsilon(i)$ into account. 
\begin{remark}
The construction in the bandit setting is such that all arms are potentially 
``good'' in the sense used in Corollary \ref{corr:reduc}, and hence 
$\sum_{j=1}^{K}\epsilon(j)^2$ coincides with $\sum_{i\in G_t}\epsilon(j)^2$
(recall $G_t$ is the set of ``good'' arms at time $t$). If one wishes to consider a situation 
where some arms $j$ are ``bad'', and obtain a bound dependent on 
$\sum_{j\in G_t}\epsilon(j)^2$, one can simply pick some sufficiently 
large values $m(j)$ for them, and ignore their contribution to the regret in 
the lower bound analysis. 
\end{remark}
The lower bound leaves open the possibility of removing the dependence on 
$K\epsilon_t(j_t)^2$ in the upper bound. This term is immaterial when 
$K\epsilon_t(j_t)$ is comparable to, or smaller than $\max_{i\in 
G_t}\epsilon_t(i)^2$ (e.g., if most arms are good, and $\epsilon_t(i)$ is 
about the same for all $i$), but there are certainly situations where it could 
be otherwise. This question is left to future work.

\section{Smooth losses}\label{sec:laplace}
As discussed in the introduction, a line of work in the online learning 
literature considered the situation where the losses of each arm varies slowly 
across time (e.g., $|\ell_t(i)-\ell_{t'}(i)|$ tends to be small when $t$ and $t'$ are close to each other), and showed how to attain better regret guarantees in 
such a case. An orthogonal question is whether such improved performance is 
possible when the losses vary smoothly \emph{across arms}. Namely, 
$|\ell_t(i)-\ell_t(i')|$ tends to be small for all pairs $i,i'$ of actions that are similar to each other.

It turns out that this assumption can be exploited, avoiding the lower bound of~\cite{gertor16}, if the learner is given (or can compute) an ``anchor point'' $a_t$ at the end of the round
$t$, 
which equals the loss of \emph{some} arm at round $t$, independent of 
the learner's randomness at that round. Importantly, 
the learner need not even know which arm has this loss. For example, it is 
often reasonable to assume that there is always some arm which attains a minimal
loss of $0$, or some arm which attains a maximal loss of $1$. In that case, 
instead of estimating losses $\ell_{t}(i)$ in $[0,1]$, it is enough to estimate 
losses of the form $\ell_{t}(i)+(1-a_t)$, which may lie in a much narrower 
range if $|\ell_{t}(i)-a_t|$ is constrained to be small. 

To see why this ``anchor point'' side-information circumvents the lower bound 
of \cite{gertor16}, we briefly discuss their construction (in a slightly 
simplified manner): The 
authors consider 
a situation where the losses are generated stochastically and independently at 
each round according to  
$\ell_t(i)=\mathrm{clip}_{[0,1]}\left(Z_t-\Delta\mathbf{1}_{i=i^*}\right)$, with 
$Z_t$ being a standard Gaussian random variable,  
$\Delta=\Theta(\sqrt{K/T})$, and $i^*$ being some arm chosen uniformly at 
random. Hence, at every round, arm $i^*$ has a loss smaller by 
$\Theta(\sqrt{K/T})$ than all other arms. Getting an expected 
regret smaller than $\Omega(\sqrt{KT})$ would then amount to detecting $i^*$. 
However, since the learner observes only a \emph{single} loss every round, the 
similarity of the losses for different arms at a given round does not help 
much. In contrast, if the learner had access to the loss $a_t$ of any fixed arm 
(independent of the learner's randomness), she could easily detect $i^*$ in 
$\Ocal(K)$ rounds, simply by maintaining a ``feasible set'' $\mathcal{I}$ of 
possible arms, picking arms $i\in \mathcal{I}$ at random, and removing it from 
$\mathcal{I}$ if $\ell_t(i)-a_t$ is positive. This process ends once 
$\mathcal{I}$ contains a single arm, which must be $i^*$.

To formalize this setting in a flexible manner, we follow a graph-based 
approach, inspired by \cite{valko2014spectral}. Specifically, we assume that at 
every round $t$, a graph over the $K$ arms, 
with an associated Laplacian matrix $L_t$ and parameter $C_t\geq 0$, can be defined so that 
the loss vector $\bloss_t$ satisfies
\[
\boldsymbol{\ell}_t^\top L_t \boldsymbol{\ell}_t = \sum_{(i,j)\in 
E_t}\big(\ell_t(i)-\ell_t(j)\big)^2\leq C_t^2~.
\]
The smaller is $C_t$, the more similar are the losses, on average. This can 
naturally interpolate between the standard bandit setting (where the losses 
need not be similar) and the extreme case where all losses are the same, in 
which case the regret is always trivially zero. Crucially, 
note that the 
learner \emph{need not} have explicit knowledge of neither $L_t$ nor $C_t$: In 
fact, our regret upper bounds, which will depend on these entities, will hold 
for any $L_t$ and $C_t$ which are valid with respect to the vectors of actual 
losses (possibly the ones minimizing the bounds). The only thing we do expect 
the learner to know (at the end of each round $t$) is the ``anchor 
point'' 
$a_t$ as described above. We also 
note that this setting is quite distinct from the graph bandits setting of 
\cite{mannor2011bandits,alon2014nonstochastic}, 
which also assumes a graph structure over the bandits, but this graph encodes 
what feedback the learner receives, as opposed to encoding similarities between 
the losses themselves.

We now turn to describe the algorithm and associated regret bound. The 
algorithm itself is very simple: Run a standard multiarmed bandits algorithm 
suitable for our setting (in particular, Exp3 \cite{auer2002nonstochastic}) 
using the shifted losses $\tilde{\ell}_t(i)=\ell_t(i)+1-a_t$. The associated 
regret guarantee is formalized in the following theorem.
\begin{theorem}\label{th:upper}
Assume that in each round $t$, after choosing $I_t$ the learner is told a number $a_t$ 
chosen 
by the oblivious adversary and such that there exists some arm $k_t$ with 
$\loss_t(k_t) = a_t$. Then Exp3 performing updates based on loss vectors 
$\tbloss_t = \bloss_t + (1-a_t)\bone$ achieves
\begin{align*}
\E\left[\sum_{t=1}^{T} \loss_t(I_t)\right] - \min_{i=1,\dots,K}\sum_{t=1}^{T} \loss_t(i)  
\le
\frac{\log K}{\eta} + \frac{\eta}{2} \sum_{t=1}^{T}\left(1 + 
\frac{C_t^2}{\lambda_2(L_t)}\right)
\end{align*}
where each $L_t$ is the Laplacian of any simple and connected graph on 
$\{1,\dots,K\}$ such that
$
\bloss_t^{\top}L_t\bloss_t \le C_t^2
$ for all $t=1,\dots,T$.
\end{theorem}
The proof is based on Euclidean-norm 
regret bounds for the Exp3 algorithm, combined with a careful analysis of the 
associated quantities based on the Laplacian constraint $\bloss_t^\top 
L_t\bloss_t\leq C_t^2$.

By this theorem, we get that if the step size $\eta$ is chosen optimally 
(based on $T,C_t,\lambda_2(L_t)$), 
then we get a regret bound of order 
$\sqrt{(\log K)\sum_{t=1}^{T}\left(1+\frac{C_t^2}{\lambda_2(L_t)}\right)}$. We note 
that even if some of these parameters are unknown in advance, this can be 
easily handled using doubling-trick arguments (see the appendices for a proof), and the same holds for our other results.
%


The bound of Theorem~\ref{th:upper} is not fully satisfying as it does not 
vanish when $C_t = 0$ (which assuming the graph is connected, implies that all 
losses are the same). The reason is 
that we need to add $1$ to each loss component in order to guarantee that we do 
not end up with negative components when $\loss_t(k_t)\bone$ is subtracted from 
$\bloss_t$. This is avoided when in each round $t$, the revealed loss 
$\loss_t(k_t)$ is the smallest component of $\bloss_t$, as formalized in the 
following corollary.
\begin{corollary}\label{cor:upper}
Assume that in each round $t$, after choosing $I_t$ the learner is told $a_t = 
\min_i\loss_t(i)$. Then Exp3 performing updates using losses $\tbloss_t = 
\bloss_t - a_t\bone$ achieves
\begin{align*}
\E\left[\sum_{t=1}^{T} \loss_t(I_t)\right] - \min_{i=1,\dots,K}\sum_{t=1}^{T} \loss_t(i)  
\le
\frac{\log K}{\eta} + \frac{\eta}{2}\sum_{t=1}^{T}\frac{C_t^2}{\lambda_2(L_t)}
\end{align*}
where each $L_t$ is the Laplacian of any simple and connected graph on 
$\{1,\dots,K\}$ such that
$
\bloss_t^{\top}L_t\bloss_t \le C_t^2
$ for all $t=1,\dots,T$.
\end{corollary}
We leave the question of getting such a bound, without $a_t$ being the smallest loss, as an open problem.

We now show how the bounds stated in Theorem~\ref{th:upper} and Corollary~\ref{cor:upper} relate to the standard Exp3 bound, which in its tightest form is of order $\sqrt{(\log K)\sum_t\norm{\bloss_t}^2}$ ---see \lemref{lem:exp3} in the supplementary material. Recall that our bounds are achieved for all choices of $L_t,C_t$ such that $\bloss_t^{\top}L_t\bloss_t \le C_t^2$ for all $t$. Now assume, for each $t$, that $L_t$ is the Laplacian of the $K$-clique. Then $L_t$ has all nonzero eigenvalues equal to $K$, and so the condition $\bloss_t^{\top}L_t\bloss_t \le C_t^2$ is satisfied for $C_t = \norm{\bloss_t}\sqrt{K}$. As $\lambda_2(L_t)$ is also equal to $K$, we have that $C_t^2\big/\lambda_2(L_t) = \norm{\bloss_t}^2$. Hence, when $\eta$ is tuned optimally (e.g., through the doubling trick), the bounds of Theorem~\ref{th:upper} and Corollary~\ref{cor:upper} take, respectively, the form
\begin{equation}\label{eq:cormin}
\sqrt{(\log K)\sum_{t=1}^{T}\min\left\{\norm{\bloss_t}^2,1+\frac{C_t^2}{\lambda_2(L_t)}\right\}}
\quad\text{and}\quad
\sqrt{(\log K)\sum_{t=1}^{T}\min\left\{\norm{\bloss_t}^2,\frac{C_t^2}{\lambda_2(L_t)}\right\}}~.
\end{equation}
Finally, we show that for fixed graphs $L_t=L$, the regret bound in~\eqref{eq:cormin} (right-hand side) is tight in the worst-case up to log factors.
\begin{theorem}\label{thm:laplow}
There exist universal constants $c_1,c_2$ such that the following holds: For 
any randomized algorithm, any $C > 0$, any $\lambda\in (0,1]$, and any sufficiently large $K$ and $T$, there 
exists a $K$-node graph with Laplacian $L$ satisfying $\lambda_2(L)\in 
[c_1\lambda,c_2\lambda]$ and an adversary strategy $\bloss_1,\dots,\bloss_T$, such that the expected 
regret (w.r.t.~the algorithm's internal randomization) is at least 
\[
\Omega\left(\min\left\{\sqrt{K},\frac{C}{\sqrt{\lambda_{2}(L)}}\right\}\sqrt{T}\right)
\]
while $\bloss_t^{\top} L \bloss_t \le C$ for all $t=1,\dots,T$.
\end{theorem}
This theorem matches \eqref{eq:cormin}, assuming that $C_t=C,L_t=L$ for all 
$t$, and that $\lambda_2(L)=\Ocal(1)$. Note that the latter assumption is 
generally the interesting regime for $\lambda_2(L)$ (for example, 
$\lambda_2(L)\leq 
1$ as long as there is \emph{some} node connected by a single edge). The proof 
is based on considering an ``octopus'' graph, composed of long threads 
emanating from one central node, and applying a standard bandit lower bound 
strategy on the nodes at the ends of the threads.

\subsection{Multiple connected components}\label{app:multcon}
The previous results of this section need the graph represented by $L_t$ to be
connected, in order for the guarantees to be non-vacuous. This is not just an 
artifact of the analysis: If the graph is not connected, at least some arms can 
have losses which are arbitrarily different than other arms, and the 
anchor point side information is not necessarily useful. Indeed, if there are multiple connected components, then $\lambda_2=0$ and our bounds become trivial.
Nevertheless, we now show it is still possible to get improved regret performance in some cases, as long as the learner is provided with anchor point information on each connected component of the graph.

We assume that at every round $t$, there is some graph defined over the arms, with 
edge set $E_t$. However, here we assume that this graph may have multiple 
connected components (indexed by $s$ in some set $\Ccal_t$). For each connected 
component $s$, with associated Laplacian $L_t(s)$, we assume the learner has access to an anchor point $m_{t}(s)$. Unlike the case discussed previously, here the 
anchor points may be different at different components, so a simple shifting of 
the losses (as done in \secref{sec:laplace}) no longer suffices to get a good 
bound. However, the anchor points still allow us to compute some 
interval, in which each loss must lie, which in turn can be plugged into the 
algorithmic reduction presented in \secref{sec:meps}. This is formalized in the 
following lemma.
\begin{lemma}
\label{l:mcc}
For any connected component $s\in \Ccal_t$, and any arm $i$ in that component, 
$
\big|\ell_t(i)-m_{t}(s)\big| \leq C_{t}\Big/\sqrt{\lambda_2\big(L_{t}(s)\big)}
$.
\end{lemma}
Based on this lemma, we know that any arm at any connected component $s$ has 
values in 
\[
\left[m_{t}(s)-\frac{C_t}{\sqrt{\lambda_2\big(L_{t}(s)\big)}},m_{t}(s)
+\frac{C_t}{\sqrt{\lambda_2\big(L_{t}(s)\big)}}\right].
\]
Using this and applying Corollary~\ref{corr:reduc}, we have the following result.
\begin{theorem}
For any fixed arm $j$, the algorithm described in Corollary~\ref{corr:reduc} satisfies
\[
\E\left[\sum_{t}\ell_t(I_t)-\sum_t \ell_t(j)\right]~\leq~ 
\frac{\log(K)}{\eta}+\frac{\eta}{2}\sum_{t=1}^{T}\left(
\frac{C_t^2}{\lambda_2\big(L_{t}(s_{\min})\big)}+\sum_{s\in 
G_t}\frac{C_t^2}{\lambda_2\big(L_{t}(s)\big)}
N_t(s)\right)
\]
where $N_{t}(s)$ is the number of arms in connected component $s$, and 
$s_{\min}$ is 
a connected component $s$ for which 
$m_{t}(s)-C_t\Big/\sqrt{\lambda_2\big(L_{t}(s)\big)}$ is smallest.
\end{theorem}
This allows us to get results which depend on the Laplacians $L_{t}(s)$, even 
when these sub-graphs are disconnected. We note however that this theorem does 
not recover the results of \secref{sec:laplace} when there is only one 
connected component, as we get 
$
	\frac{\log K}{\eta}+\frac{\eta(K+1)}{2}\sum_{t=1}^{T}\frac{C_t^2}{\lambda_2(L_t)}
$ where the $K+1$ factor is spurious. The reason for this looseness is that we 
go through a coarse upper bound on the magnitude of the losses, and lose the 
 dependence on the Laplacian along the way. This is not just an 
 artifact of the analysis: Recall that the algorithmic reduction proceeds by 
 using transformations of the actual losses, and these transformations may not 
 satisfy the same Laplacian constraints as the original losses. Getting a 
 better algorithm with improved regret performance in this particular setting 
 is left to future work. 


\bibliographystyle{plain}
\bibliography{mybib}


\appendix

\section{Proof of \thmref{thm:reduc}}
The proof consists mainly of proving \eqref{eq:reduc}. The in-expectation bounds follows by applying expectations on both sides of the inequality, and noting that conditioned on rounds $1,\ldots,t-1$, the conditional expectation of $\ell_t(I_t)$ equals $\sum_{i=1}^{K}p_t(i)\ell_t(i)$, and the conditional expectation of $\tilde{\ell}_t(\tilde{I}_t)$ equals $\sum_{i=1}^{K}\tilde{p}_t(i)\tilde{\ell}_t(i)$. Also, the statement on the range of each $\tilde{\ell}_t(i)$ is immediate from the definition of $\tilde{\ell}_t(i)$ and \eqref{eq:good} below.

We now turn to prove \eqref{eq:reduc}. By adding and subtracting terms, it is sufficient to prove that
\begin{align}
\nonumber
	\sum_{t} \sum_i & \big( p_t(i)\ell_t(i)-m_t(j_t)+\epsilon_t(j_t)\big)-\sum_t \big(\ell_t(a)-m_t(j_t)+\epsilon_t(j_t)\big)
\\ &\leq
\label{eq:reducproof}
	\sum_{t,i}\tilde{p}_t(i)\tilde{\ell}_t(i)-\sum_t \tilde{\ell}_t(a)~.
\end{align}
We will rely on the following facts, which are immediate from the definition of good and bad arms: Any bad arm $i$ must satisfy 
\begin{equation}\label{eq:bad}
\ell_t(i) \geq m_t(j_t)+\epsilon_t(j_t)
\end{equation}
and any good arm $i$ must satisfy
\begin{equation}\label{eq:good}
m_t(j_t)-\epsilon_t(j_t) \leq \ell_t(i) \leq m_t(j_t)+\epsilon_t(j_t)+2\epsilon_t(i)~.
\end{equation}
Based on this, we have the following two claim, whose combination immediately implies~\eqref{eq:reducproof}.

\textbf{Claim 1.} For any fixed arm $a$, $\tilde{\ell}_t(a) \leq \ell_t(a)-m_t(j_t)+\epsilon_t(j_t)$.

To show Claim~1, we consider separately the case where $a$ is a bad arm at round $t$, and where $a$ a good arm at round $t$. If $a$ is a bad arm, then $\tilde{\ell}_t(a)=2\epsilon_t(j_t)$, which is at most $\ell_t(a)-m_t(j_t)+\epsilon_t(j_t)$ by \eqref{eq:bad}. Otherwise, if $a$ is a good arm at round $t$, the observation follows by definition of $\tilde{\ell}_t$.

\textbf{Claim 2.}
\[
	\sum_{i=1}^{K}p_t(i)\ell_t(I_t)-m_t(j_t)+\epsilon_t(j_t)\leq 
	\sum_{i=1}^{K}\tilde{p}_t(i)\tilde{\ell}_t(i)
\]
where
		\[
		p_t(i)=\begin{cases} \tilde{p}_t(i)& \text{$i\neq j_t$ and $i$ is good}\\
		0 & \text{$i\neq j_t$ and $i$ is bad}\\
		\tilde{p}_t(j_t)+\sum_{\text{$i$ is bad}}p(i)&i=j_t. \end{cases}
		\]
	To show Claim~2, recall that if $i$ is a good arm, then $\tilde{\ell}_t(i) = 
	\ell_t(i)-m_t(j_t)+\epsilon_t(j_t)$, and otherwise, we have  
	$\tilde{\ell}_t(I_t) = 2\epsilon_t(j_t)\geq 
	\ell_t(j_t)-m_t(j_t)+\epsilon_t(j_t)$ (since $\ell_t(j_t)\leq 
	m_t(j_t)+\epsilon_t(j_t)$ by definition). Letting $G_t$ denote the set of 
	good arms at round $t$, we have:
	\begin{align*}
	\sum_i \tilde{p}_t(i)\tilde{\ell}_t(i)&= \sum_{i\in 
	G_t}\tilde{p}_t(i)\tilde{\ell}_t(i)+\sum_{i~\text{bad}}\tilde{p}_t(i)\tilde{\ell}_t(i)\\
	 &= \sum_{i\in 
	G_t}\tilde{p}_t(i)\big(\ell_{t}(i)-m_t(j_t)+\epsilon_t(j_t)\big)+\sum_{i~\text{is bad}}\tilde{p}_t(i)\,2\,\epsilon_t(j_t)\\
	&\geq \sum_{i\in 
	G_t}\tilde{p}_t(i)\big(\ell_{t}(i)-m_t(j_t)+\epsilon_t(j_t)\big)+\sum_{i~\text{is bad}}\tilde{p}_t(i)\big(\ell_{t}(j_t)-m_t(j_t)+\epsilon_t(j_t)\big)\\
	&= \sum_i p_t(i)\ell_t(I_t)-m_t(j_t)+\epsilon_t(j_t)~.
	\end{align*}
Combining the two claims above, and summing over $t$, we get \eqref{eq:reducproof} as required.

\section{Proof of \thmref{thm:epslowbound}}
Suppose the learner uses some (possibly randomized) strategy, and let $A$ be a 
random variable denoting its random coin flips. Our goal is to provide lower 
bounds on 
\[
\sup_{\bloss_1,\dots,\bloss_T}
\left(\E_A\left[\sum_{t=1}^{T}\ell_t(I_t)\right]-\min_{j=1,\dots,K} \sum_{t=1}^{T}\ell_t(j)\right)
\]
where the expectation is with respect to the learner's (possibly randomized) 
strategy.
Clearly, this is lower bounded by 
\[
\E_{J,L}\E_A A\left[\sum_{t=1}^{T}\ell_t(I_t)-\sum_{t=1}^{T}\ell_t(J)\right],
\]
where $\E_{J,L}$ signifies expectation over some distribution over indices $J$ and losses $\{\ell_i(t)\}$. By Fubini's theorem, this equals
\[
\E_A \E_{J,L}\left[\sum_{t=1}^{T}\ell_t(I_t)-\sum_{t=1}^{T}\ell_t(J)\right] 
~\geq~ \inf_A 
\E_{\{\ell_i(t)\}_{i,t},j}\left[\sum_{t=1}^{T}\ell_t(I_t)-\sum_{t=1}^{T}\ell_t(j)\right],
\]
where $\inf_A$ refers an infimum over the learner's random coin flips. Thus, we 
need to provide some distribution over indices $J$ and losses, so that for any 
\emph{deterministic} learner,
\begin{equation}\label{eq:boundtoshow}
 \E \left[\sum_{t=1}^{T}\ell_t(I_t)-\sum_{t=1}^{T}\ell_t(J)\right]
\end{equation}
is lower bounded as stated in the theorem.

The proof will be composed of two constructions, depending on whether we are in 
the bandit of full information setting, and whether $\frac{\max_j 
\epsilon(j)^2}{\sum_{j}\epsilon(j)^2}$ is larger or smaller than $1/4$.

\subsection{The case $\frac{\max_j \epsilon(j)^2}{\sum_{j}\epsilon(j)^2}\leq 
\frac{1}{4}$ with bandit feedback}

For this case, we will consider the following distribution: Let $J$ be distributed on $\{1,\ldots,K\}$ according to the probability distribution $p(1),\ldots,p(k)$ (to be specified later). Conditioned on any $J=j$, we define the distribution over losses as follows, independently for each round $t$ and index $i$:
\begin{itemize}
	\item If $i\neq j$, then $\ell_t(i)$ equals $\max_r\epsilon(r)+\epsilon(i)$ w.p, $\frac{1}{2}$, and $\max_r\epsilon(r)-\epsilon(i)$ w.p. $\frac{1}{2}$.
	\item If $i=j$, then $\ell_t(i)$ equals $\max_r\epsilon(r)+\epsilon(i)$ w.p. $\frac{1-\delta(i)}{2}$, and $\max_r\epsilon(r)-\epsilon(i)$ w.p. $\frac{1+\delta(i)}{2}$.
\end{itemize}
Also, let $\E_j,\Pr_j$ denote expectation and probabilities (over the space of possible losses and indices) conditioned on the event $J=j$. With this construction, we note that $\E_j[\ell_t(j)] = \max_r\epsilon(r)-\delta(j)\delta(j)$, and $\E_j[\ell_t(i)]=\max_r\epsilon(r)$ if $i\neq j$. As a result,
\[
\E_j[\ell_t(I_t)-\ell_t(j)]={\Pr}_{j}(I_t\neq 
j)\cdot\E_j[\ell_t(I_t)-\ell_t(j)|I_t\neq j)] = {\Pr}_j(I_t\neq 
j)\epsilon(j)\delta(j),
\]
and therefore \eqref{eq:boundtoshow} equals
\begin{align}
\sum_{j=1}^{K}p(j)\E_j\left[\sum_{t=1}^{T}(\ell_t(I_t)-\ell_t(j))\right] &~=~
\sum_{j=1}^{K}p(j)\sum_{t=1}^{T}{\Pr}_j(I_t\neq 
j)\epsilon(j)\delta(j)\notag\\&~=~ 
\sum_{j=1}^{K}p(j)\epsilon(j)\delta(j)\sum_{t=1}^{T}\left(1-{\Pr}_j(I_t=j)\right).\label{eq:toshow1}
\end{align}

Let $\Pr_{0}$ denote the probability distribution over $\{\ell_t(i)\}_{t,i}$, 
where for \emph{any} $i$ and $t$, $\ell_t(i)$ is independent and equals $\max_r 
\epsilon(r)\pm \epsilon(i)$ with equal probability (note that this induces a 
probability on any event which is a deterministic function of the loss 
assignments, such as $I_t=j$ for some $t,j$). By a standard 
information-theoretic argument (see for instance \cite[proof of Lemma 
3.6]{bubeck2012regret}), we have that
\[
|{\Pr}_j(I_t=j)-{\Pr}_0(I_t=j)| ~\leq~ \sqrt{\frac{\E_0[T(j)]}{2}\cdot 
KL\left(\frac{1}{2}\middle|\middle|\frac{1-\delta(j)}{2}\right)},
\]
where $T(j)$ is the number of times arm $j$ was chosen by the learner, and 
$KL\left(\frac{1}{2}\middle|\middle|\frac{1-\delta(j)}{2}\right)=\frac{1}{2}\log\left(\frac{1}{1-\delta^2(j)}\right)$
 is the Kullback-Leibler divergence between Bernoulli distributions with $1/2$ 
and $(1-\delta(j))/2$. Using the easily-verified fact that $\log(1/(1-z))\leq 
2z$ for all $z\in [0,1/2]$, it follows that
\[
|{\Pr}_j(I_t=j)-{\Pr}_0(I_t=j)| ~\leq~ \sqrt{\frac{\E_0[T(j)]\delta^2(j)}{2}},
\]
as long as $\delta^2(j)\leq 1/2$. Plugging this back into \eqref{eq:toshow1}, we get the lower bound
\begin{equation}\label{eq:toshow2}
\sum_{j=1}^{K}p(j)\epsilon(j)\delta(j)\sum_{t=1}^{T}\left(1-{\Pr}_0(I_t=j)-\sqrt{\frac{\E_0[T(j)]\delta^2(j)}{2}}\right)~,
\end{equation}
which is valid as long as $\max_j \delta^2(j)\leq 1/2$. 

Now, for all $j=1,\ldots,K$, we pick 
\[
p(j)=\frac{\epsilon(j)^2}{\sum_j \epsilon(j)^2}~~,~~ \delta(j) = 
\mathbf{1}_{\epsilon(j)>0}\cdot 
\frac{\sqrt{\sum_{j}\epsilon(j)^2}}{\epsilon(j)\sqrt{T}},
\]
(assuming that $\max_j \delta^2(j)\leq 1/2$). Plugging back to \eqref{eq:toshow2}, and letting $J=\{j\in \{1,\ldots,k\}:\epsilon_j>0\}$, we get
\begin{align*}
\sum_{j\in J}&\frac{\epsilon(j)^2}{\sqrt{T\sum_j \epsilon(j)^2}}\cdot 
\sum_{t=1}^{T}\left(1-{\Pr}_0(I_t=j)-\sqrt{\frac{\E_0[T(j)]\sum_j\epsilon(j)^2}{2\epsilon(j)^2T}}\right)\\
&=\sqrt{T\sum_j \epsilon(j)^2}\cdot \sum_{j\in 
J}\frac{\epsilon(j)^2}{T\sum_j\epsilon(j)^2} 
\sum_{t=1}^{T}\left(1-{\Pr}_0(I_t=j)-\sqrt{\frac{\E_0[T(j)]\sum_j\epsilon(j)^2}{2\epsilon(j)^2T}}\right)\\
&\geq \sqrt{T\sum_j \epsilon(j)^2}\left(1-\frac{\max_j 
\epsilon(j)^2}{\sum_j\epsilon(j)^2}-\frac{1}{T}\sum_{t=1}^{T}\left(\sum_{j\in 
J}\frac{\epsilon(j)^2}{\sum_{j\in J} \epsilon(j)^2}\cdot 
\sqrt{\frac{\E_0[T(j)]\sum_j\epsilon(j)^2}{2\epsilon(j)^2T}}\right)\right)\\
&\geq \sqrt{T\sum_j \epsilon(j)^2}\left(1-\frac{\max_j 
\epsilon(j)^2}{\sum_j\epsilon(j)^2}-\sqrt{\frac{\sum_{j\in 
J}\E_0[T(j)]}{2T}}\right)\\
&\geq \sqrt{T\sum_j \epsilon(j)^2}\left(1-\frac{\max_j 
\epsilon(j)^2}{\sum_j\epsilon(j)^2}-\sqrt{\frac{1}{2}}\right),
\end{align*}
where in the second-to-last step we used the fact that $\sum_{j\in 
J}\frac{\epsilon(j)^2}{\sum_{j\in J \epsilon(j)^2}}\sqrt{a_j} \leq 
\sqrt{\sum_{j\in J}\frac{\epsilon(j)^2}{\sum_{j\in J \epsilon(j)^2}}a_j}$ for 
any non-negative $a_j$, which follows from Jensen's inequality and the fact 
that $\frac{\epsilon(j)^2}{\sum_{j\in J}\epsilon(j)^2}$ represents a 
probability distribution over the indices in $J$. Since we assume that 
$\frac{\max_j\epsilon(j)^2}{\sum_{j\in J}\epsilon(j)^2}\leq \frac{1}{4}$, the 
above is at least
$
0.04\sqrt{T\sum_j \epsilon(j)^2},
$
so we get overall that
\[
 \E \left[\sum_{t=1}^{T}\ell_t(I_t)-\sum_{t=1}^{T}\ell_t(J)\right] ~\geq~ 
 0.04\sqrt{T\sum_j \epsilon(j)^2},
\]
under the assumption that $\frac{\max_j\epsilon(j)^2}{\sum_{j\in 
J}\epsilon(j)^2}\leq \frac{1}{4}$ and that $T$ is sufficiently large so that 
$\max_j   \frac{\mathbf{1}_{\epsilon(j)>0} 
\sum_{j}\epsilon(j)^2}{\epsilon(j)^2T} ~\leq~ \frac{1}{2}$. Note that the 
latter condition indeed holds under the theorem's conditions.

\subsection{The case $\frac{\max_j \epsilon(j)^2}{\sum_{j}\epsilon(j)^2}\geq 
\frac{1}{4}$ with bandit feedback or with full information feedback}
We now turn to consider either the full information setting, or the bandit 
setting when $\frac{\max_j \epsilon(j)^2}{\sum_{j}\epsilon(j)^2}\geq 
\frac{1}{4}$. In the latter case, we note that $\sum_j \epsilon(j)^2$ is at 
most a constant factor larger than $\max_j \epsilon(j)^2$, so it is sufficient 
to prove a lower bound of $c\sqrt{T\max_j\epsilon(j)^2}$ for some universal 
positive $c$. In fact, we will prove this lower bound regardless of the values 
of $\epsilon(1),\ldots,\epsilon(K)$, and even in the easier full information 
case. Therefore, the same construction will give us a lower bound for both the 
full information setting, and the bandit setting when $\frac{\max_j 
\epsilon(j)^2}{\sum_{j}\epsilon(j)^2}\geq \frac{1}{4}$.

To lower bound \eqref{eq:boundtoshow}, we will use the following distribution over losses and $J$, letting $i_{\max}$ be some arbitrary index in $\arg\max_{i\in \{1,\ldots,K\}}\epsilon(i)$, and $\delta\in (0,1/2]$ be some parameter to be chosen later:
\begin{itemize}
	\item For any $t=1,\ldots,T$ and $i\neq i_{\max}$, we fix $\ell_t(i)=\epsilon(i_{\max})$.
	\item We pick a value $z$ uniformly at random from $\{-1,1\}$. Then, for all $t=1,\ldots,T$, we let $\ell_t(i_{\max})$ equal $2\epsilon(i_{\max})$ with probability $\frac{1-z\delta}{2}$, and $0$ with probability $\frac{1+z\delta}{2}$. Also, if $z=1$, we let $J=1$, and if $z=-1$, we let $J=2$.
\end{itemize}
Clearly, this loss assignment is valid (as $|\ell_t(i)-\epsilon(i_{\max})|\leq 
\epsilon(i)$ for all $t,i$). Intuitively, we let all arms but $i_{\max}$ have a 
fixed loss of $1/2$, and randomly choose $i_{\max}$ to be either a ``good'' arm 
or a ``bad'' arm compared to the other arms (with expected value 
$(1-z\delta)\epsilon(i_{\max})$, which can be wither 
$(1+\delta)\epsilon(i_{\max})$ or $(1-\delta)\epsilon(i_{\max})$). By letting 
$\delta=\Theta(1/\sqrt{T})$, we ensure that the algorithm cannot distinguish 
between these two events, and therefore will ``err'' and pick 
$\Omega(\epsilon(i_{\max})/\sqrt{T})$-suboptimal arms with at least constant 
probability throughout the $T$ rounds, hence incurring 
$\Omega(\epsilon(i_{\max})\sqrt{T})$ regret.

To make this more formal, let $\E_+,{\Pr}_+$ denote expectations and probabilities conditioned on $z=1$, and $\E_-,{\Pr}_-$ denote expectations and probabilities conditioned on $z=-1$. With this notation,  \eqref{eq:boundtoshow} can be written as
\begin{align}
&\frac{1}{2}\cdot\E_+\left[\sum_{t=1}^{T}\left(\ell_t(I_t)-\ell_t(1)\right)\right]+\frac{1}{2}\cdot\E_-\left[\sum_{t=1}^{T}\left(\ell_t(I_t)-\ell_t(2)\right)\right]\notag\\
&=\frac{1}{2}\sum_{t=1}^{T}{\Pr}_+(I_t\neq 
1)\cdot\delta\epsilon(i_{\max})+\frac{1}{2}\sum_{t=1}^{T}{\Pr}_-(I_t=1)\cdot\delta\epsilon(i_{\max})\notag\\
&=\frac{\delta\epsilon(i_{\max})}{2}\sum_{t=1}^{T}\left(1-{\Pr}_+(I_t=1)+{\Pr}_-(I_t=1)\right)\notag\\
&\geq 
\frac{\delta\epsilon(i_{\max})}{2}\sum_{t=1}^{T}\left(1-\left|{\Pr}_+(I_t=1)-{\Pr}_-(I_t=1)\right|\right)\label{eq:toshow3}
\end{align}
Noting that $I_t$ (as a random variable) depends only on the random loss 
assignments of arm $1$, and applying Pinsker's inequality, we have that 
$\left|{\Pr}_+(I_t=1)-{\Pr}_-(I_t=1)\right|\leq 
\sqrt{\frac{1}{2}KL(P^+_t||P^-_t)}$, where $KL(P^+_t||P^-_t)$ is the 
Kullback-Leibler divergence between the distributions of the losses of arm $1$ 
in rounds $1,2,\ldots,t-1$, under $z=-1$ and under $z=1$. Since the losses are 
independent across rounds, we can apply the chain rule and get that
\[
\left|{\Pr}_+(I_t=1)-{\Pr}_-(I_t=1)\right|~\leq~ 
\sqrt{\frac{1}{2}KL(P^+_t||P^-_t)}~\leq~ 
\sqrt{\frac{t-1}{2}KL\left(\frac{1-\delta}{2}\middle|\middle|\frac{1+\delta}{2}\right)},
\]
where $KL\left(\frac{1-\delta}{2}\middle|\middle|\frac{1+\delta}{2}\right)$ is the Kullback-Leibler divergence between Bernoulli distributions with parameters $\frac{1-\delta}{2}$ and $\frac{1+\delta}{2}$. This in turn equals 
\[
\sqrt{\frac{t-1}{2}\cdot \delta\log\left(\frac{1+\delta}{1-\delta}\right)}~\leq~ \sqrt{\frac{3T}{2}\delta^2},
\]
where we used the easily-verified fact that $\log\left(\frac{1+\delta}{1-\delta}\right)\leq 3\delta$ for all $\delta\in (0,1/2]$. Plugging this back into \eqref{eq:toshow3}, we get overall that
\[
\E\left[\sum_{t=1}^{T}\left(\ell_t(I_t)-\ell_t(J)\right)\right] ~\geq~ 
\frac{\delta\epsilon(i_{\max})}{2}\cdot 
T\left(1-\delta\sqrt{\frac{3T}{2}}\right)
\]
Picking $\delta = 1/2\sqrt{T}$ (which is valid since it is in $(0,1/2]$ for all 
$T$), we get a lower bound of $c\epsilon(i_{\max})\sqrt{T}= \sqrt{T\max_j 
\epsilon(j)^2}$ for some positive $c$ as required.

\section{Proof of \thmref{th:upper}}
We start by recalling the classical analysis of the Exp3 regret.
\begin{lemma}
\label{lem:exp3}
For losses $\loss_t(i) \in [0,1]$, the regret of the Exp3 algorithm run with parameter $\eta > 0$ satisfies
\[
\E\left[\sum_{t=1}^{T} \loss_t(I_t)\right] - \min_{k=1,\dots,K}\sum_{t=1}^{T} \loss_t(k)
\le
\frac{\log K}{\eta} + \frac{\eta}{2}\sum_{t=1}^T \norm{\bloss_t}^2~.
\]
\end{lemma}
\begin{proof}
The proof is as follows,
\begin{align*}
\frac{W_{t+1}}{W_t}
&=
\sum_{i=1}^K \frac{w_{t+1}(i)}{W_t}
\\&=
\sum_{i=1}^K \frac{w_{t}(i)}{W_t}\,\exp\bigl(-\eta\,\hloss_t(i)\bigr)
\\&=
\sum_{i=1}^K p_{t}(i)\exp\bigl(-\eta\,\hloss_t(i)\bigr)
\\ &\le
\sum_{i=0}^K p_{t}(i)\left(1 - \eta\,\hloss_t(i) + \frac{\bigl(\eta\,\hloss_t(i)\bigr)^2}{2}\right)\\
& \text{(using $e^{-x} \leq 1-x+x^2/2$ for all $x \ge 0$)}
\\ &\le
1 - \eta\sum_{i=1}^K p_{t}(i)\hloss_t(i) + \frac{\eta^2}{2}\sum_{i=1}^K p_{t}(i)\hloss_t(i)^2~.
\end{align*}
Taking logs, upper bounding, and summing over $t=1,\dots,T$ yields
\[
\log\frac{W_{T+1}}{W_1}
\le
- \eta\sum_{t=1}^T \sum_{i=1}^K p_{t}(i)\hloss_t(i)
+ \frac{\eta^2}{2} \sum_{t=1}^T\sum_{i=1}^K p_{t}(i)\hloss_t(i)^2~.
\]
Moreover, for any fixed comparison arm $k$, we also have
\[
\log\frac{W_{T+1}}{W_1}
\ge
\log\frac{w_{T+1}(k)}{W_1} = -\eta\sum_{t=1}^T \hloss_t(k) - \log K~.
\]
Putting together,
\begin{align}
\label{eq:exp3-basic}
\sum_{t=1}^{T} \sum_{i=1}^K p_{t}(i) \hloss_t(i) - \sum_{t=1}^{T} \hloss_t(k)  
\le
\frac{\log K}{\eta} + \frac{\eta}{2} \sum_{t=1}^{T} \sum_{i=1}^K p_{t}(i) 
\hloss_t(i)^{2}~.
\end{align}
Next, note that
\begin{equation}
\label{eq:standard}
\E_t\Bigl[\hloss_t(i)\Bigr] = \loss_t(i)
\qquad\text{and}\qquad
\E_t\Bigl[\hloss_t(i)^2\Bigr] = \frac{\loss_t(i)^2}{p_{t}(i)}~.
\end{equation}
This immediately gives
\begin{align*}
\E\left[\sum_{t=1}^{T} \loss_t(I_t)\right] - \sum_{t=1}^{T} \loss_t(k)
\le
\frac{\log K}{\eta} + \frac{\eta}{2}\sum_{t=1}^T \sum_{i=1}^K \loss_t(i)^2
\end{align*}
concluding the proof.
\end{proof}

A simple graph is an unweighted, undirected graph, containing no self-loops or multiple edges. In the following, let $\bone = (1,\dots,1)$.
\begin{lemma}
\label{l:minor}
Let $G = (V,E)$ be a simple and connected graph with $|V|=K$ nodes, and let $L$ be its Laplacian matrix. Let $L(i,i)$ be the $(K-1)\times (K-1)$ submatrix obtained by deleting the $i$-th row and the $i$-th column from $L$. Then, for any $i=1,\dots,K$, the eigenvalues of $L(i,i)$ are the non-zero eigenvalues of $L$. 
\end{lemma}
\begin{proof}
$L$ has $K-1$ non-zero eigenvalues because it is connected. Let $\bv = (v_1,\dots,v_K)$ be an eigenvector of $L$ with eigenvalue $\lambda > 0$. Since $L\bone = \bzero$, we have that
\[
\lambda^2 = \bv^{\top}L\bv = \big(\bv - v_i\bone\big)^{\top}L\big(\bv - v_i\bone\big) = \bv_i^{\top}L(i,i)\bv_i
\]
where $\bv_i = \big(v_1-v_i,\dots,v_{i-1}-v_i,v_{i+1}-v_i,\dots,v_K-v_i\big)$. Hence $\bv_i$ is an eigenvector of $L(i,i)$ with eigenvalue $\lambda$.
\end{proof}
We are now ready to prove \thmref{th:upper}.
\begin{proof}[Proof of \thmref{th:upper}]
Using the invariance of the regret to translation of the losses and the fact that $\loss_t(i)+1-a_t\ge 0$ for all $t$, $i$, and $k_t$,
\begin{align}
\E\left[\sum_{t=1}^{T} \loss_t(I_t)\right] - \sum_{t=1}^{T} \loss_t(k)
&=
\E\left[\sum_{t=1}^{T} \tloss_t(I_t)\right] - \sum_{t=1}^{T} \tloss_t(k)
\notag\\ &\le
\frac{\log K}{\eta} + \frac{\eta}{2}\sum_{t=1}^T 
\norm{\tbloss_t}^2~.\label{eq:actreg}
\end{align}
Since $\bloss_t^{\top}L_t\bloss_t = \tbloss_t^{\top}L_t\tbloss_T$, and since 
$\tbloss_t$ has component $k_t$ equal to $1$, we can upper bound each term of 
the summation in the right-hand side by the solution of the convex program
\begin{align}
\begin{array}{ll}
    {\displaystyle \max_{\bloss\in\reals^K}\norm{\bloss}^2}
\\
    \text{such that} & \bloss^{\top}L_t\bloss \le C_t^2
\\                   & \exists i\in\{1,\dots,K\} \; \loss(i) = 1~.
\end{array}\label{eq:convprog}
\end{align}
Using Lemma~\ref{l:minor}, the above program is equivalent to
\[
\begin{array}{ll}
    {\displaystyle \max_{\bloss\in\reals^{K-1}}\Big(1 + \norm{\bloss}^2\Bigr)}
\\
    \text{such that} & \bloss^{\top}L_t(1,1)\bloss \le C_t^2
\end{array}
\]
where $L(1,1)$ is full rank. Hence we can set $\bu = L_t(1,1)^{1/2}\bloss/C_t$ 
and obtain the equivalent program
\begin{equation}
\label{eq:fiedler}
 1 + \max_{\bu\in\reals^{K-1} \,:\, \norm{\bu} \le 1} 
 C_t^2\big(\bu^{\top}L_t(1,1)^{-1}\bu\big) = 1 + \frac{C_t^2}{\lambda_2(L_t)}
\end{equation}
which gives us the claimed bound.
\end{proof}

\section{Proof of \thmref{thm:laplow}}
To prove the theorem, let $d=\lceil 1/\lambda\rceil$, and $k$ be any integer 
such that $d$ divides $k-1$. Finally, define $G_{k,d}$ to be an ``Octopus'' 
graph composed of $(k-1)/d$ tentacles of equals length $d$. Formally, for any 
two nodes $i,j\in \{1,\ldots,k\}$ where $j>i$ w.l.o.g., we have $(i,j)\in E$ if 
and only if
\[
((j=k)~~\text{and}~~(i=1~\text{mod}~d))~~~\text{or}~~~((j\neq k)~~\text{and}~~(j=i+1)~~\text{and}~~(i\neq 0~\text{mod}~d))
\]
Note that here, node $k$ is the ``central'' node, from which all tentacles emanate (first tentacle corresponding to nodes $1,2,\ldots,d$, second tentacle corresponding to nodes $d+1,d+2,\ldots,2d$ and so on). 

The theorem is a straightforward corollary of the following two lemmas.
\begin{lemma}
For an Octopus graph $G_{k,d}$ with Laplacian $L_{k,d}$, for any $C > 0$, and for any randomized algorithm, there exists an adversary strategy $\bloss_1,\bloss_2,\dots$ such that the expected regret is $\Omega\left(\min\left\{\sqrt{k},Cd\right\}\sqrt{T}\right)$ while $\bloss_t^{\top}L_{k,d}\bloss_t \le C$ for each $t=1,\dots,T$.
\end{lemma}
\begin{proof}[Proof Sketch]
In the graph, there are $\Omega(k)$ points at a distance $\Omega(d)$ from the center. These ``faraway'' points can get loss magnitudes as large as $\frac{1}{2}\pm \min\{1,Cd/\sqrt{k}\}$, while satisfying the budget constraints and having the central point as an anchor with a fixed loss of $1/2$  (to make sure budget constraint is satisfied, assign losses in increments of roughly $C/\sqrt{k}$ along each tentacle, so points in the faraway half of each points gets losses varying as $\pm\min\{1,Cd/\sqrt{k}\}$). Specifically, all those faraway points will get the same random Gaussian loss every round (equalling $1/2$ in expectation), except for one point whose loss will always be a $\Theta(\sqrt{k/T})$ smaller. Therefore, the regret lower bound is order of 
\[
\Omega\left(\min\left\{1,\frac{Cd}{\sqrt{k}}\right\}\sqrt{kT}\right) = \Omega\left(\min\left\{\sqrt{k},Cd\right\}\sqrt{T}\right).
\]
\end{proof}
\begin{lemma}
For an Octopus graph $G_{k,d}$, $\lambda_2=\Theta(1/d^2)$ (where $\Theta(\cdot)$ hides universal constants).
\end{lemma}
\begin{proof}
We begin with the lower bound. By~\cite[Theorem~1]{rad2011lower}, $\lambda_2$ is lower bounded by $k/C_{\max}$, where $C_{\max}=\max_{e\in E}C_e$, and $C_e$ is the sum, over all pairs of distinct nodes $i$ and $j$ of the length of the shortest path between $(i,j)$ passing through $e$ assuming this path exists. For the graph as defined above, any edge separates at most $d$ nodes from at most $k-1$ other nodes, and the length of the path between any two nodes is at most $2d$. Therefore, $C_{\max}\leq 2kd^2$, so
\[
\lambda_2\geq \frac{k}{2kd^2} = \frac{1}{2d^2} = \Omega\left(\frac{1}{d^2}\right).
\]

Turning to the upper bound, by corollary 4.4 in \cite{grone1990laplacian}, for any tree graph of diameter $D$, 
\[
\lambda_2 \leq 2\left(1-\cos(\pi/(D+1))\right),
\]
and since $\cos(x)\geq 1-x^2/2$ for all $x$, this implies
\[
\lambda_2 \leq \frac{\pi^2}{(D+1)^2}.
\]
An Octopus graph $G_{k,d}$ is a tree with diameter $2d$, hence
\[
\lambda_2 \leq \frac{\pi^2}{(2d+1)^2} = \Ocal\left(\frac{1}{d^2}\right),
\]
from which the result follows.
\end{proof}

\section{Proof of \lemref{l:mcc}}
For simplicity, we will drop the $s,t$ subscripts, as they play no role here. 
The proof follows by an analysis similar to that of \eqref{eq:convprog}, where 
$1$ is replaced by $m_{t}(s)$ and noting that the $2$-norm upper bounds the 
$\infty$-norm. Specifically, making the worst-case assumption that the 
adversary budget $C_t$ is spent solely on the connected component we are 
concerned with, we need to solve the convex program
\[
\begin{array}{ll}
{\displaystyle \max_{\bloss\in\reals^K}\norm{\bloss-m\mathbf{1}}_\infty}
\\
\text{such that} & \bloss^{\top}L\bloss \le C^2
\\                   & \exists i\in\{1,\dots,K\} \; \loss(i) = m~.
\end{array}
\]
which is equivalent (using the fact that $\bloss^\top L\bloss$ is invariant to 
shifting the coordinates of $\bloss$) to
\[
\begin{array}{ll}
{\displaystyle \max_{\bloss\in\reals^K}\norm{\bloss}_\infty}
\\
\text{such that} & \bloss^{\top}L\bloss \le C^2
\\                   & \exists i\in\{1,\dots,K\} \; \loss(i) = 0~.
\end{array}
\]	
Upper bounding the $\infty$-norm by the $2$-norm, and using 
Lemma~\ref{l:minor}, the above program is equivalent to
\[
\begin{array}{ll}
{\displaystyle \max_{\bloss\in\reals^{K-1}} \norm{\bloss}_2}
\\
\text{such that} & \bloss^{\top}L(1,1)\bloss \le C^2
\end{array}
\]
where $L(1,1)$ is full rank. Hence we can set $\bu = L(1,1)^{1/2}\bloss/C$ and 
obtain the equivalent program
\[
\max_{\bu\in\reals^{K-1} \,:\, \norm{\bu} \le 1} 
C\sqrt{\big(\bu^{\top}L(1,1)^{-1}\bu\big)} = \frac{C}{\sqrt{\lambda_2(L)}}
\]
which gives us the claimed bound.	

\section{Optimal tuning of $\eta$ in Theorem~\ref{th:upper}}
In this section, we show that using Exp3 with the the doubling trick we obtain 
an expected regret scaling as $\sqrt{(\log K)\sum_t\norm{\tilde{\ell}_t}^2}$. 
As all our results depend on upper bounding $\norm{\tilde{\ell}_t}^2$, the same 
bounds will hold for \thmref{th:upper}.

We apply the doubling trick as follows. Let $\eta_r = \sqrt{(2\log K)/2^r}$ for 
each $r=r_0,r_0+1,\dots$ where $r_0 = \big\lceil \log_2\log K + 1\big\rceil$ is 
chosen so that $\eta_r \le 1$ for all $r \ge r_0$. Let $T_r$ the random set of 
consecutive time steps when the same $\eta_r$ was used. Exp3 starts using $\eta 
= \eta_{r_0}$ and monitors the observable random quantity
\[
Q_s = \sum_{i=1}^K p_s(i) \hloss_s(i)^{2}~.
\]
Whenever $\sum_{t \in T_r} Q_t > 2^r$ is detected while Exp3 is running with 
$\eta = \eta_r$, Exp3 is restarted with $\eta = \eta_{r+1}$.
\begin{corollary}
	If Exp3 is run with the above doubling trick, then its regret satisfies
	\begin{align*}
	\E\left[\sum_{t=1}^{T} \sum_{i=1}^K p_{t}(i) \hloss_t(i) \right] &- 
	\sum_{t=1}^{T}\hloss_t(k)
	\\ &\le
	\big\lceil\log_2\big(T+4\log K\big)\big\rceil + 5\sqrt{2(\log 
	K)\left(\sum_{t=1}^T\norm{\bloss_t}^2 + 4\log K\right)}~.
	\end{align*}
\end{corollary}
\begin{proof}
	Let $\bar{Q}_t = Q_1+\cdots+Q_s$. The largest $r$ we need is the smallest 
	$R$ such that
	\[
	\sum_{r=r_0}^R 2^r \ge \bar{Q}_T
	\]
	and so $R = \big\lfloor\log_2(\bar{Q}_T + 4\log K)\big\rfloor$. Therefore
	\[
	\sum_{r=r_0}^R 2^{r/2} < 5\sqrt{\bar{Q}_T+4\log K}~.
	\]
	Because of \eqref{eq:exp3-basic},
	\[
	\sum_{t=1}^{T} \left( \sum_{i=1}^K p_{t}(i) \hloss_t(i) - \hloss_t(k) 
	\right)  
	\le
	\frac{\log K}{\eta} + \frac{\eta}{2} \sum_{t=1}^{T} \sum_{i=1}^K p_{t}(i) 
	\hloss_t(i)^{2}
	\]
	and so
	\[
	\sum_{t \in S_r} \left( \sum_{i=1}^K p_{t}(i) \hloss_t(i) - \hloss_t(k) 
	\right)
	\le
	\frac{\log K}{\eta_r} + \frac{\eta_r}{2} \sum_{t \in S_r} Q_t
	\le
	\sqrt{2(\log K)2^r}~.
	\]
	Since a regret of at most $1$ is incurred whenever Exp3 is restarted, we 
	have
	\begin{align*}
	\E\left[\sum_{t=1}^{T} \sum_{i=1}^K p_{t}(i) \hloss_t(i) \right] &- 
	\sum_{t=1}^{T}\hloss_t(k)
	\\ &\le
	\E\Big[\big\lceil\log_2\big(\bar{Q}_T+4\log K\big)\big\rceil\Big] + 
	5\,\E\left[\sqrt{2(\log K)\big(\bar{Q}_T+4\log K\big)}\right]
	\\&\le
	\big\lceil\log_2\big(T+4\log K\big)\big\rceil + 5\sqrt{2(\log 
	K)\left(\sum_{t=1}^T\norm{\bloss_t}^2 + 4\log K\right)}
	\end{align*}
	where in the last step we used Jensen's inequality and \eqref{eq:standard}.
\end{proof}

\end{document}